\documentclass[twoside]{article}

\usepackage{natbib}
\usepackage{hyperref}
\usepackage{url}

\usepackage{xcolor}

\usepackage{graphicx}
\usepackage{algorithm}
\usepackage{algpseudocode}
\usepackage{amsmath,amsthm,amssymb,amsfonts,dsfont,bm}
\usepackage{subcaption}
\usepackage{float}
\usepackage{enumitem}

\usepackage[capitalize,noabbrev]{cleveref}
\usepackage{booktabs}

\allowdisplaybreaks

\theoremstyle{plain}
\newtheorem{theorem}{Theorem}[section]

\newtheorem{lemma}[theorem]{Lemma}

\theoremstyle{definition}

\newtheorem{assumption}[theorem]{Assumption}
\theoremstyle{remark}

\def\R{\mathbb{R}}
\def\E{\mathbb{E}}

\newcommand{\sign}{\mbox{sign}}

\newcommand{\<}{\left\langle}
\renewcommand{\>}{\right\rangle}

\usepackage[accepted]{aistats2026}
\begin{document}

%

%
\runningauthor{Wei Shen, Zhang Yaxiang, Minhui Huang, Mengfan Xu, Jiawei Zhang, Cong Shen}

\twocolumn[

\aistatstitle{MLorc: Momentum Low-rank Compression for Memory Efficient Large Language Model Adaptation}

\aistatsauthor{ Wei Shen\footnotemark \addtocounter{footnote}{-1} \And Zhang Yaxiang  \footnotemark \footnotemark  \addtocounter{footnote}{-1} \And  Minhui Huang}

\aistatsaddress{ University of Virginia \And  National University of Singapore \And Meta }

\aistatsauthor{
Mengfan Xu \And Jiawei Zhang \And Cong Shen }

\aistatsaddress{
University of Massachusetts Amherst \And University of Wisconsin-Madison \And University of Virginia} 
]

\footnotetext{Equal contribution. The first two authors are listed in alphabetical order.}
\addtocounter{footnote}{1}
\footnotetext{Corresponding author. Address correspondance at \texttt{e1353410@u.nus.edu}.}

\begin{abstract}
With the increasing size of large language models (LLMs), full-parameter fine-tuning imposes substantial memory demands. To alleviate this, we propose a novel memory-efficient training paradigm called \textbf{M}omentum \textbf{Lo}w-\textbf{r}ank \textbf{c}ompression (MLorc).
The key idea of MLorc is to compress and reconstruct the momentum of matrix parameters during training to reduce memory consumption. Compared to LoRA, MLorc avoids enforcing a fixed-rank constraint on weight update matrices and thus enables full-parameter learning. Compared to GaLore, MLorc directly compress the momentum rather than gradients, thereby better preserving the training dynamics of full-parameter fine-tuning. We provide a theoretical guarantee for its convergence under mild assumptions. Empirically, MLorc consistently outperforms other memory-efficient training methods, matches or even exceeds the performance of full fine-tuning at small ranks (e.g., \(r=4\)), and generalizes well across different optimizers, all while not compromising time or memory efficiency. 
\end{abstract}

\section{INTRODUCTION}

Large Language Models (LLMs) have demonstrated strong generalization capabilities on downstream tasks after fine-tuning \citep{liu2019roberta,raffel2020exploring,li2021prefix}. However, full-parameter fine-tuning is prohibitively expensive in terms of GPU memory. In addition to storing billions of model parameters and activation values, training also requires memory for gradients and various optimizer states (e.g., first- and second-order momentum terms in Adam). Without memory-saving techniques, standard AdamW consumes approximately three times more memory for gradients and optimizer states than for storing model parameters alone.

One promising approach to reduce this memory overhead is to design memory-efficient optimization paradigms tailored for fine-tuning, where the objective is to adapt the model to specific tasks. LoRA (Low-Rank Adaptation) \citep{hu2022lora} is one of the most widely adopted parameter-efficient fine-tuning (PEFT) methods: it freezes the original model weights and introduces trainable, low-rank updates. However, LoRA inherently limits the space of possible weight updates due to its low-rank constraint, and its reparameterization can significantly alter training dynamics. Prior studies have shown that LoRA may underperform full-parameter fine-tuning on certain tasks \citep{biderman2024lora,xia2024chain} and exhibit distinct update patterns \citep{liu2024dora}.

Recently, GaLore (Gradient Low-Rank Projection) \citep{zhao2024galore}, another memory-efficient optimization approach, has garnered attention. GaLore projects gradients and optimizer states (e.g., momentum) into a low-dimensional subspace for storage and uses the same projector to reconstruct optimizer states used to update weight. The projectors are periodically updated through Singular Value Decomposition (SVD) on stochastic gradient matrices. GaLore claims to overcome the limitations of low-rank methods like LoRA by canceling low-rank factorization and improving training dynamics. Nevertheless, both prior research \citep{luo2024badam} and our experiments reveal that GaLore usually underperforms, even compared to LoRA. We attribute GaLore's underperformance to its suboptimal training dynamics, that is, the reconstruction in GaLore can break the structure of the momentum, due to the instability of singular vectors of the stochastic gradient. We will analyze this in detail in \Cref{section: method}.

To address these challenges, we propose a new memory-efficient training paradigm, \textbf{M}omentum \textbf{Lo}w-\textbf{r}ank \textbf{c}ompression (MLorc). 
Unlike GaLore, MLorc directly compresses and reconstructs momentum instead of gradients with Randomized SVD (RSVD) \citep{halko2011finding} and then uses these compressed momentum to run some benchmark optimizers like Adam \citep{diederik2014adam} or Lion \citep{chen2023symbolic}, thereby maintaining closer alignment with the training dynamics of full-parameter fine-tuning. The motivation of designing MLorc stems from our empirical observation that, during LLM fine-tuning, the momentum of matrix parameters often exhibits an approximately low-rank structure, implying that compressing the momentum does not result in significant information loss. Moreover, directly compressing the momentum can avoid the affect of unstable stochastic gradient.

Our main contributions can be summarized as follows:
\begin{itemize}[noitemsep,topsep=0pt,leftmargin = *]
    \item We propose a new memory-efficient training paradigm called \textbf{M}omentum \textbf{Lo}w-\textbf{r}ank \textbf{c}ompression (MLorc). The key idea of MLorc is to compress and reconstruct the momentum of matrix parameters during training to reduce memory consumption. The main motivation of MLorc is from our empirical observation of approximately low-rank structure of matrix parameters' momentum.
    \item We provide a theoretical convergence guarantee for MLorc with the Lion optimizer \citep{chen2023symbolic}, matching the original Lion's sample complexity \citep{dong2024convergence} under mild assumptions.
    \item We validate the effectiveness of MLorc through extensive experiments across diverse models, datasets, and optimizers. Our results demonstrate that MLorc outperforms LoRA \citep{hu2022lora}, GaLore \citep{zhao2024galore}, and LDAdamW \citep{robert2024ldadam} on math and coding tasks with LLaMA2-7B, and achieves the highest average performance on GLUE tasks with RoBERTa-Base \citep{liu2019roberta}. At the same time, MLorc maintains competitive runtime and memory efficiency compared to other memory-saving methods.
\end{itemize}

\section{RELATED WORK}\label{section: related work}
\textbf{Low-rank adaptation.} Low-rank adaptation methods, such as LoRA \citep{hu2022lora}, have been proposed to enhance the memory efficiency of fine-tuning large language models. LoRA introduces trainable low-rank matrices into each layer of a pre-trained model, significantly reducing the number of trainable parameters while maintaining performance comparable to full fine-tuning. Inspired by LoRA, Flora \citep{hao2024flora} periodically resamples random projection matrices during training to compress the gradients, aiming to achieve higher-rank updates over time while maintaining same level memory consumption. There are also other variants of LoRA designed for improving performance and other purposes \citep{meng2024pissa,kalajdzievski2023rank,dettmers2023qlora,hayou2024lora+,zhang2023adalora,li2024mixlora,zi2023delta,wang2023multilora,li2024mixlora,zhang2023lora}. 

In contrast, GaLore \citep{zhao2024galore} is a memory-efficient fine-tuning method that reduces the storage cost of gradients and optimizer states by projecting them into a dynamically learned low-rank subspace. 
There are also other variants of GaLore designed for improving time efficiency and further reducing memory footprint \citep{zhang2024q,rajabi2025subtrack,yang2025sparse}. However, a recent study \citep{he2024subspace} shows that GaLore does not always converge to the optimal solution under standard assumptions. \cite{he2024subspace} proposed GoLore, a variant of GaLore that utilized random low-rank projection instead of the greedy one used in GaLore. Recent studies have also introduced GaLore-inspired variants. For example, Fira \citep{chen2024fira} improves performance by combining the exact gradient with the GaLore update. LDAdam \citep{robert2024ldadam} incorporates a projection-aware update rule for optimizer states together with a generalized error-feedback mechanism, explicitly addressing the compression of both gradients and optimizer states.

\textbf{Memory-efficient optimization.} There are also other techniques to reduce memory footprint during training, including gradient checkpointing \citep{chen2016training}, quantization \citep{dettmers8,li2023memory} and other memory-efficient optimization methods (AdaLomo \citep{lv2023adalomo}, MeZO \citep{malladi2023fine}, etc). These methods address different aspects of the memory bottleneck. They are orthogonal to our methods and some of them can be combined with MLorc to further reduce the memory footprint.

\textbf{Matrix compression.} Matrix compression techniques, particularly those based on Singular Value Decomposition (SVD), play an important role in model compression \citep{wang2025dobi,liu2024eora} and reducing memory footprint \citep{zhao2024galore} in model training. Randomized SVD (RSVD) \citep{halko2011finding} is an efficient variant of SVD. SVD decomposes a matrix into low-rank components that preserve most of its information, while RSVD accelerates this process by approximating the dominant singular subspace using random projections. These methods enable compact representation of gradients and optimizer states, thus laying the foundation of our method.

\section{PRELIMINARIES AND MLORC}\label{section: method}

In this section, we introduce our main method. We begin by presenting the preliminaries on which our method is built, including LoRA and GaLore. We then formally elaborate on MLorc, covering the algorithmic framework and steps, memory analysis, and convergence analysis.

\subsection{Preliminaries}

\textbf{Notation.}
For a matrix $A\in \R^{m\times n}$, we denote its Frobenius norm  as $\|A\|_F$, denote its entrywise $l_1$ norm as $\|A\|_{1,1}\triangleq \sum_{i=1}^m\sum_{j=1}^n |A_{ij}|$. Given a batch sample $B=\{\xi^i\}_{i=1}^b$, we denote $\nabla f(W; B)=\frac{1}{b}\sum_{i=1}^b \nabla f(W; \xi_i)$.

\subsubsection{LoRA}
LoRA \citep{hu2022lora} is a parameter-efficient fine-tuning technique designed for adapting large pre-trained models to downstream tasks. Instead of updating the full model weights, LoRA freezes the original parameters and injects trainable low-rank matrices into specific layers (typically attention or feedforward layers). This significantly reduces the number of trainable parameters and memory requirements during fine-tuning. Initial weight of the model \(W_0 \in R^{m \times n}\) is frozen, and weight update is achieved by updating two low rank matrices: \(B \in R^{m \times r}\) and \(A \in R^{r \times n}\), typically \( r \ll m,n\), as illustrated in the following formula:
\begin{align}
    W = W_{0} + BA.
\end{align}

Despite its memory efficiency, LoRA has several limitations. First, the imposed low-rank constraint can restrict the expressiveness of weight updates, potentially limiting performance on tasks that require more complex adaptations. Second, LoRA introduces a reparameterization of the weight update process, which alters the training dynamics and can lead to suboptimal convergence in some scenarios \citep{zhao2024galore,meng2024pissa}. Empirical studies have shown that LoRA can underperform full fine-tuning on certain tasks \citep{biderman2024lora,xia2024chain}.

\subsubsection{GaLore}

GaLore \citep{zhao2024galore} is a recent memory-efficient training paradigm designed to reduce the memory footprint of optimizer states and gradients during fine-tuning of large language models. Unlike LoRA, which freezes the model and injects low-rank trainable adapters into the weight matrices, GaLore applies a low-rank projection directly to the gradients and optimizer states. Specifically, it performs periodic SVD on the gradients to identify a low-rank subspace, into which the optimizer states (e.g., momentum, variance) are projected. This strategy allows GaLore to maintain full-parameter weight updates while significantly compressing the memory required for training, aiming to preserve training dynamics more faithfully than LoRA’s reparameterized updates. 

However, there is still room for improvement in GaLore. Although GaLore does not constrain the weight updates themselves to be low-rank, it relies on fixed (over a certain number of steps) low-rank projections, which may still limit its ability to fully capture dynamic gradient information. 

To be specific, in step \(t\), GaLore (on Adam) first gets projector \(P_t\): it is updated every \(T\) steps using the singular vector of gradient \(G_t\); otherwise \(P_t\) is equal to \(P_{t-1}\). Subsequently, GaLore projects \(G_t\): \(R_t = P_t^{T}G_t\) and first/second order momentum \(M_t, V_t\) is constructed by exponential average of \(R_t\), just like original Adam. Finally, low-rank update \(N_t=\frac{M_t}{\sqrt{V_t}+\epsilon}\) and GaLore uses \(P_t\) to project back \(N_t\).  To ensure GaLore's training dynamics align with those of full-parameter training, it implicitly depends on two key assumptions: (1) gradients exhibit a low-rank structure, and (2) the eigenspace of \(N_t\) can be properly recovered by pre-defined projectors. 

While the first assumption is well-supported by prior studies \citep{zhaozero,cosson2023low}, the second is questionable in the context of mini-batch training. Infrequent projector updates can result in misaligned projections and reconstructions, and even with expensive high-frequency updates, a critical limitation remains: \textbf{there exists no well-defined projector for back-projection of \(N_t\)}, since momentum is an accumulation (i.e., weighted average) of mini-batch gradients across different steps. For example, with default \( \beta_2 =0.999\), gradients of 100 steps earlier still have comparable weight with current gradients in second-order momentum; hence, $M_t$ and $V_t$'s eigen space is very different from $g_{t-\tau}$'s for any \( \tau\). Also, \(N_t\) is a non-linear transformation of $M_t$ and $V_t$, which makes the preservation of eigenspace impossible. Consequently, \(N_t\)'s eigenspace cannot be recovered from any single-step gradient's eigenspace. This motivates us to shift focus from gradient compression to momentum compression -- \textbf{directly compressing and reconstructing momentum rather than gradients}.

\subsection{MLorc}\label{section: mlorc}

\subsubsection{Algorithm and Implementation}

\begin{figure}[htbp]
  \centering
    \includegraphics[width=22em]{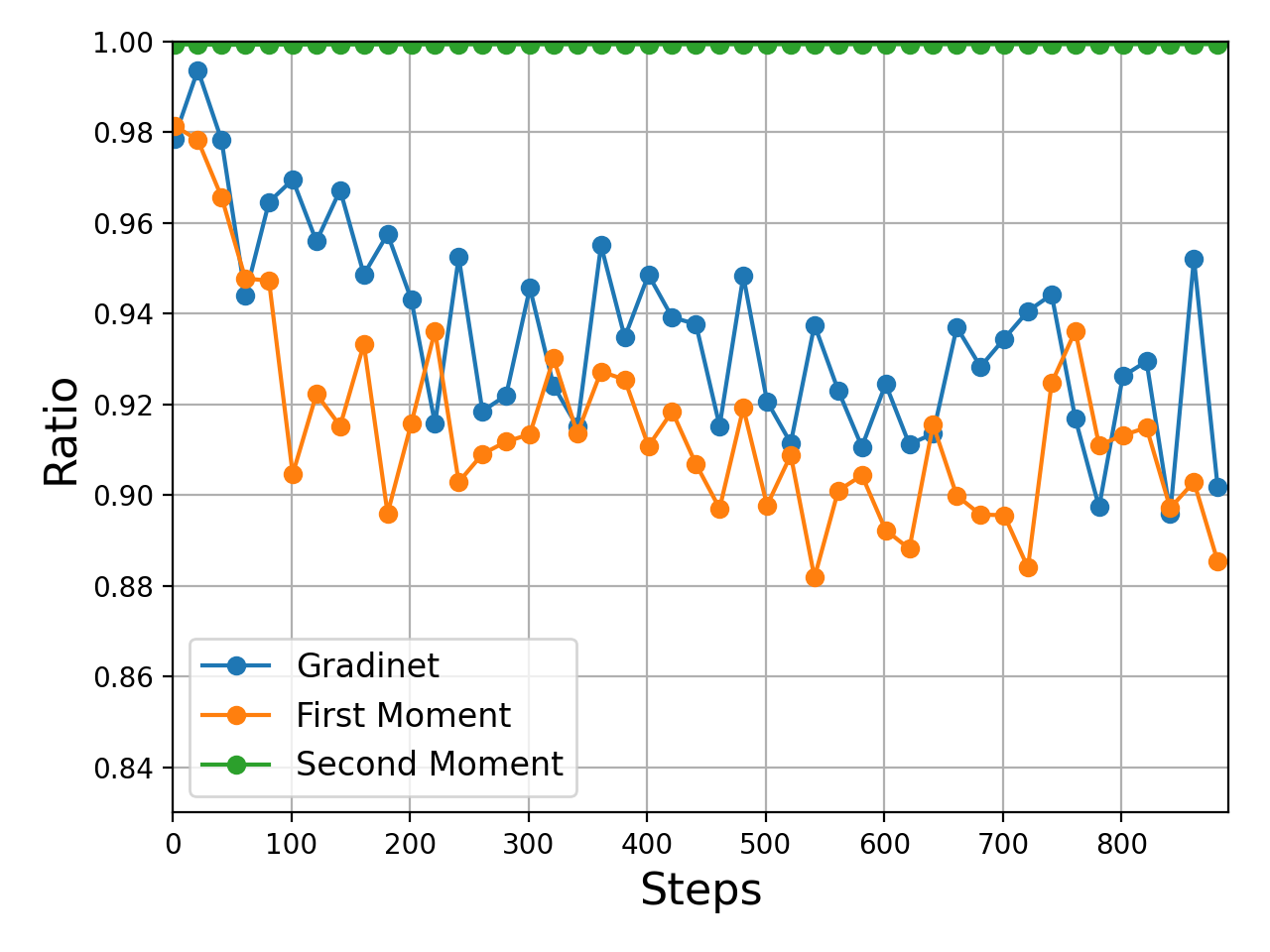}
  \caption{Ratio of top-8 singular values to total singular values for gradient, first moment, and second moment during AdamW finetuning of RoBERTa-base on the STSB dataset.}
  \label{fig: moment stsb}
\end{figure}

To enable compression at the momentum level, we first investigate whether momentum exhibit low-rank structure. We analyze various components involved in optimization by examining the concentration of singular values in gradients and momenta. For a matrix $A\in R^{m\times n}$ (suppose $m\geq n\geq 8$) with singular values $\sigma_1\geq \sigma_2\geq \cdots \geq \sigma_n$, we calculate its ratio of top-8 singular values as $\frac{\sum_{i=1}^8 \sigma_i}{\sum_{i=1}^n \sigma_i}$. We use this ratio to quantify the concentration of singular values (i.e., the 'low-rankness') of a matrix, where a larger value reflects a stronger low-rank structure.  We compute its ratio for the gradient, first moment, and second moment matrices during AdamW fine-tuning of RoBERTa-base on the STSB dataset and get Figure~\ref{fig: moment stsb}.  As illustrated in Figure~\ref{fig: moment stsb}, the first-order momentum shows a spectral pattern similar to that of the gradients, while the second-order momentum demonstrates an even stronger low-rank structure. More experimental evidence can be found in \Cref{app: low rank}. Motivated by these empirical observations, along with our earlier analysis of GaLore’s training dynamics, we propose \textbf{M}omentum \textbf{Lo}w-\textbf{r}ank \textbf{c}ompression (MLorc), a new memory-efficient training paradigm for large-scale model fine-tuning.

The core idea of MLorc is to efficiently and accurately compress momentum for storage and reconstruct momentum to update weight, and we chose RSVD \citep{halko2011finding} to do this. A detailed introduction to RSVD is deferred to \Cref{app: rsvd}. Here, we highlight a key property: the time complexity of RSVD is \(O(mnr)\), which is on the same order as the projection and back-projection operations. Notably, MLorc can be applied to any optimizer (e.g, Adam, Lion) with momentum. Taking AdamW as an example: at each optimization step, we first reconstruct the first and second order momentum  \(\tilde{m}_{t-1}, \tilde{v}_{t-1}\)  from the compressed optimizer state:  \( \tilde{m}_{t-1} = m_{u,t-1} m_{s,t-1} m_{v,t-1}^\top, \quad \tilde{v}_{t-1} = v_{u,t-1} v_{s,t-1} v_{v,t-1}^\top\),  update them using the current gradient:  \(m_t = \beta_1 \tilde{m}_{t-1} + (1 - \beta_1) g_t, \quad v_t = \beta_2 \tilde{v}_{t-1} + (1 - \beta_2) g_t^2\)  , and then compress the updated momenta using RSVD:  $(m_{u,t}, m_{s,t}, m_{v,t}) = \text{RSVD}(m_t)$, $(v_{u,t}, v_{s,t}, v_{v,t}) = \text{RSVD}(v_t)$.  Finally, we use these updated momenta to perform the usual parameter update.

There is also a special consideration for second-order momentum, which must remain entry-wise non-negative. A straightforward approach is to apply an entry-wise \(ReLU\) to the reconstructed second-order momentum  \(\tilde{v}_{t-1}\).  However, this introduces zeros in the reconstructed values, and since \( \beta_{2}\) is typically set very close to 1, these zeros can result in extremely small values in the updated second-order momentum. This can destabilize training and degrade model performance. 
To address this, we add a small constant entry-wise to the zero values in $ReLU(\tilde{v}_{t-1})$. Given that parameter groups often have different scales, this constant should be chosen adaptively. In practice, we set it to the absolute mean of the negative part of the reconstructed momentum, which is usually much smaller than the positive part.
For example, we first calculate the absolute mean of the negative part of the reconstructed momentum $\zeta(\tilde{v}_{t-1}) = \frac{1}{|\{ i : (\tilde{v}_{t-1})_i < 0 \}|} \sum_{i :(\tilde{v}_{t-1})_i < 0} |(\tilde{v}_{t-1})_i|$. Then, we update the $\tilde{v}_{t-1}$ according to the following formula:
\begin{align}\label{eq: update v}
    \tilde{v}_{t-1} \gets ReLU(\tilde{v}_{t-1})+ \zeta (\tilde{v}_{t-1})\cdot \mathbf{1}_{\{\tilde{v}_{t-1} < 0\}},
\end{align}
where $\mathbf{1}_{\{\tilde{v}_{t-1} < 0\}}$ is the indicator vector of the negative entries of $\tilde{v}_{t-1}$.
This modification is different from adding \( \epsilon\) on the square root of second-order momentum: \(0\)s here come from the error introduced by momentum compression rather than the small magnitude of the corresponding gradient element.

\begin{algorithm}
\caption{MLorc-AdamW}
\begin{algorithmic}[1]
\State\textbf{Input}: Initial weights \( W_1 \), learning rate \( \alpha \), betas \( \beta_1, \beta_2 \), weight decay rate $\lambda$, constant \( \varepsilon \), target rank \( r \), oversampling parameter \( p \), batch size $b$.
\State Initialize RSVD factors: \( (m_{u,0}, m_{s,0}, m_{v,0}) \gets 0 \), \( (v_{u,0}, v_{s,0}, v_{v,0}) \gets 0 \), \( t \gets 1 \)
\While{not converged}
    \State Sample a mini‑batch $B_t=\{\xi_t^i\}_{i=1}^b$ uniformly at random
    \State Compute gradient: \( g_t \gets \nabla f(W_{t}; B_t)\)

    \State \( \tilde{m}_{t-1} \gets m_{u,t-1} m_{s,t-1} m_{v,t-1}^\top \)
    \State \( \tilde{v}_{t-1} \gets v_{u,t-1} v_{s,t-1} v_{v,t-1}^\top \)

    \State Update $\tilde{v}_{t-1}$ according to \eqref{eq: update v}
    
    \State \( m_t \gets \beta_1 \tilde{m}_{t-1} + (1 - \beta_1) g_t \)
    \State \( v_t \gets \beta_2 \tilde{v}_{t-1} + (1 - \beta_2) g_t^2 \)

    \State \( (m_{u,t}, m_{s,t}, m_{v,t}) \gets \text{RSVD}(m_t, r, p) \)
    \State \( (v_{u,t}, v_{s,t}, v_{v,t}) \gets \text{RSVD}(v_t, r, p) \)

    \State \( \hat{m}_t \gets \frac{m_t}{1 - \beta_1^t} \)
    \State \( \hat{v}_t \gets \frac{v_t}{1 - \beta_2^t} \)

    \State \( W_{t+1} \gets W_{t} - \alpha( \frac{\hat{m}_t}{\sqrt{\hat{v}_t} + \varepsilon} +\lambda W_{t})\)
    \State \( t \gets t + 1 \)
\EndWhile
\State \Return \( W_t \)
\end{algorithmic}
\label{alg: adam}
\end{algorithm}

\begin{algorithm}
\caption{MLorc-Lion}
\begin{algorithmic}[2]
\State \textbf{Input}: Initial weights \( W_1 \), learning rate \( \alpha \), betas \( \beta_1, \beta_2 \), target rank \( r \), oversampling parameter \( p \), batch size $b$.
\State Initialize RSVD factors: \( (m_{u,0}, m_{s,0}, m_{v,0}) \gets 0 \), \( t \gets 1 \)
\While{not converged}
    \State Sample a mini‑batch $B_t=\{\xi_t^i\}_{i=1}^b$ uniformly at random
    \State Compute gradient: \( g_t \gets \nabla f(W_{t}; B_t) \)

    \State \( \tilde{m}_{t-1} \gets m_{u,t-1} m_{s,t-1} m_{v,t-1}^\top \)

    \State \( c_t \gets \beta_1 \cdot \tilde{m}_{t-1} + (1 - \beta_1) \cdot g_t \)
    \State \( m_t \gets \beta_2 \cdot \tilde{m}_{t-1} + (1 - \beta_2) \cdot g_t \)

    \State \( (m_{u,t}, m_{s,t}, m_{v,t}) \gets \text{RSVD}(m_t, r, p) \)

    \State \( W_{t+1} \gets W_{t} - \alpha \cdot \text{sign}(c_t) \)
    \State \( t \gets t + 1 \)
\EndWhile
\State \Return \( W_t \)
\end{algorithmic}
\label{alg: lion}
\end{algorithm}

\begin{table*}[htbp]
\caption{Memory comparison. Assume $W\in R^{m\times n}$, rank $r$.}
\centering
\begin{tabular}{lcccccc}
\toprule
 & Full finetuning (AdamW) & LoRA (AdamW) & GaLore & MLorc-AdamW  \\
\midrule
Weights & $mn$ & $mn + mr + nr$ & $mn$ & $mn$ \\
Optimizer States & $2mn$ & $2mr + 2nr$  & $mr + 2nr$ & $2mr + 2nr$  \\
\bottomrule
\end{tabular}\label{tab: memory comparison}
\end{table*}

Algorithm \ref{alg: adam} shows the detailed description of MLorc-AdamW. 
Similarly, the high-level idea of MLorc can also be applied to other optimizers, such as Lion \citep{chen2023symbolic}. Since Lion only maintains a single first-order momentum, MLorc-Lion is simpler than MLorc-AdamW. It only requires reconstructing the previous momentum $\tilde{m}_{t-1}$ at each step $t$ and applying RSVD to compress the current momentum ${m}_{t}$. Algorithm \ref{alg: lion} shows a detailed description of MLorc-Lion. We will validate the effectiveness of MLorc across AdamW and Lion, and provide a theoretical convergence proof for MLorc-Lion.

\subsubsection{Memory Consumption Analysis}\label{section: memory consumption}

In \Cref{tab: memory comparison}, we provide a comparison of the memory consumption of four methods: full finetuning, LoRA, GaLore, and MLorc. For a parameter matrix of size $m \times n$, full finetuning (with AdamW) needs to store the weight $mn$ and two copies of the momentum terms $2mn$. LoRA requires storing additional weights $mr + nr$, but only needs $2mr + 2nr$ optimizer states. GaLore reduces optimizer memory to $2nr$; however, it also needs to store the projection matrix $mr$. MLorc-AdamW needs to store $m_u, m_s, m_v$ and $v_u, v_s, v_v$, and since $m_s$ and $v_s$ can be absorbed into $m_u, v_u$ (or into $m_v, v_v$), the optimizer memory of MLorc-AdamW is $2mr + 2nr$. Note that in typical fine-tuning settings (also in our experiments), $r\ll \min\{m,n\}$, which means, with the same rank $r$, LoRA, GaLore and MLorc achieve similar memory savings for the optimizer states.

GPU memory consumption during LLM training primarily comes from four sources: model weights, gradients, optimizer states, and activation values. Among these, weights and optimizer states always occupy GPU memory; therefore, in \Cref{tab: memory comparison} we primarily compare these two components. Other memory overhead varies depending on the forward and backward phases as well as the specific implementation.
LoRA typically needs to store gradients only for a small number of low-rank trainable parameters. For GaLore and MLorc, the gradients that must be stored depend on the specific implementation. Both GaLore and MLorc support per-layer weight updates \citep{lv2024full}, under which they need to store at most only one layer’s weight gradients, thus consuming very little memory. As shown in \Cref{app: memory}, under per-layer weight updates MLorc can use less memory than LoRA.
Even without per-layer weight updates, although GaLore and MLorc require more memory than LoRA for storing gradients, the overall memory usage also depends on the activations, which in turn depends on batch size. In fact, in the experiments corresponding to \Cref{tab:results} in our paper, the peak memory footprint often occurs during the forward pass, when all activation values must be stored, indicating that gradients storage does not affect the overall memory peak. Taken together, these observations suggest that MLorc can indeed achieve memory efficiency comparable to that of LoRA and GaLore.

\subsubsection{Convergence Analysis}\label{section: convergence}
In this section, we present a convergence analysis of MLorc-Lion (\Cref{alg: lion}), demonstrating that it can achieve the same sample complexity as the original Lion optimizer \citep{chen2023symbolic, dong2024convergence}. Lion is a well-known optimizer and often achieves performance comparable to AdamW when optimizing neural networks. Lion utilizes the sign function to adjust the update magnitude of each component, which is conceptually similar to AdamW. In this work, we demonstrate the theoretical guarantees of our MLorc framework by presenting the convergence analysis of MLorc-Lion. The extension to MLorc-AdamW is left for future work.
We consider optimizing loss function $f:\R^{m\times n}\to \R$. And we use the following standard assumptions. 

\begin{assumption}\label{ass: smooth}
    The loss function $f$ is $L$-Lipschitz
smooth, i.e. for any $W, W'\in \R^{m\times n}$, we have
\begin{align*}
    \|\nabla f(W)-\nabla f(W')\|_F\leq L\|W-W'\|_F.
\end{align*}
\end{assumption}

\begin{assumption}\label{ass: variance}
    $\nabla f(W;\xi)$ is an unbiased stochastic estimator of the true gradient $\nabla f(W)$ and have a bounded variance, i.e.
    \begin{align*}
        &\E[\nabla f(W;\xi)]=\nabla f(W)\\
        &\E\|\nabla f(W;\xi)-\nabla f(W)\|_F^2]\leq \sigma^2.
    \end{align*}
\end{assumption}

With these standard assumptions, we have following theorem.
\begin{theorem}[informal]\label{thm: lion}
    Under Assumptions \ref{ass: smooth} and \ref{ass: variance}, applying \Cref{alg: lion} with appropriate parameters, we have
    \begin{align*}
        \frac{1}{T}\sum_{t=1}^{T}\E[\|\nabla f(W_t)\|_{1,1}]\leq O(1)\left[\frac{\sqrt{dL\Delta}}{\sqrt{T}}+\frac{ \sigma \sqrt{d}}{\sqrt{b}}\right],
    \end{align*}
where $\Delta=f(W_1)-\inf_{W}f(W)$, $d=mn$.
\end{theorem}
The formal statement and proof of \Cref{thm: lion} can be found in \Cref{app: proof}. According to \Cref{thm: lion}, when $\sigma=0$ (deterministic case), we can find an $\epsilon$-entrywise $\ell_1$-norm stationary point of $f$ with a complexity of $O(\Delta L d \epsilon^{-2})$; when $\sigma\neq 0$ (stochastic case), with a large batch size $b = \Theta(d\sigma^2\epsilon^{-2})$, we can find an $\epsilon$-entrywise $\ell_1$-norm stationary point of $f$ with a sample complexity of $O(\Delta L d^2 \sigma^2 \epsilon^{-4})$, matching the same sample complexity as the original Lion \citep{dong2024convergence}.

\section{EXPERIMENTS}\label{section: exp}

In this section, we demonstrate the effectiveness of MLorc through extensive experiments on NLG and NLU tasks, spanning diverse models, datasets, and optimizers.\footnote{Code is available at \url{https://github.com/weishen-git/MLorc}.}

\subsection{Experiments on NLG tasks with LLaMA2-7B}\label{section: exp nlg}

In this subsection, we evaluate MLorc's performance on large language models, focusing on NLG (Natural Language Generation) tasks. We fine-tuned LLaMA 2-7B \citep{touvron2023llama} on two tasks: math and code. To demonstrate the effectiveness of MLorc, We compare MLorc's performance  with Full fine-tuning, LoRA \citep{hu2022lora}, GaLore \citep{zhao2024galore} and its variant LDAdamW\citep{robert2024ldadam} (both optimized by AdamW). We also compare MLorc's performance on MLorc-Lion \citep{chen2023symbolic} with Full Lion and LoRA (Lion) to explore its applicability to different optimizers. Experimental results suggest that MLorc significantly reduces training loss during optimization and improves validation accuracy on test datasets.

\textbf{Experimental Setup.}  For the math task, the model was fine-tuned LLaMA 2-7B on the MetaMathQA dataset \citep{yu2023metamath} and evaluated on GSM8K \citep{cobbe2021training} validation sets. For the code task, the model was fine-tuned on the CodeFeedback dataset \citep{zheng2024opencodeinterpreter} and evaluated on the HumanEval \citep{chen2021evaluating} dataset. All experiments were conducted on 1 H100-96 GPU, using subsets of training datasets containing 10K data points and were trained for 2 epochs. We use a rank=4 for all memory-efficient training paradigm, a batch size of 32 with gradient checkpointing and without gradient accumulation, a linear learning rate scheduler with a warmup ratio of 0.03. We set learning rate after tuning on each method and each dataset. It is worth mentioning that we set \(\beta_{1}\) of MLorc-AdamW as 0.8 rather than the default value 0.9, in order to mitigate the influence of approximation error arising from RSVD. More hyperparameter setups, such as specific learning rates and oversampling parameters, can be found in \Cref{app: nlg}. Average accuracy over four evaluations and the standard deviation of accuracy is reported in Table~\ref{tab:results}.

\begin{table}[htbp]
\centering
\caption{Results of LLaMA 2-7B fine-tuned on math and code tasks. MLorc consistently outperforms other memory-efficient training paradigms. Lion version of GaLore and LDAdamW are not available; LoRA is independent of the choice of optimizer.}
\label{tab:results}
\begin{tabular}{lcc}
\toprule
Method(r=4) & GSM8K & HumanEval \\
\midrule
Full (AdamW) & \(47.69_{\pm 0.15}\) & \(21.96_{\pm 0.46}\) \\
MLorc (AdamW) & \(47.37_{\pm 1.09}\) & \(20.70_{\pm 0.42}\) \\
LoRA (AdamW) & \(45.98_{\pm 0.52}\) & \(17.85_{\pm 1.07}\) \\
GaLore & \(38.89_{\pm 0.73}\) & \(17.25_{\pm 0.49}\) \\
LDAdamW & \(41.85_{\pm 0.60}\) & \(18.60_{\pm 1.08}\) \\
\\
Full (Lion) & \(46.38_{\pm 1.11}\) & \(18.00_{\pm 0.30}\) \\
MLorc (Lion) & \(47.75_{\pm 0.25}\) & \(18.75_{\pm 0.78}\) \\
LoRA (Lion) & \(45.53_{\pm 0.52}\) & \(16.00_{\pm 0.83}\) \\
\bottomrule
\end{tabular}
\end{table}

\begin{figure*}[htbp]
  \centering
  \begin{subfigure}[t]{0.48\textwidth}
    \centering
    \includegraphics[width=\linewidth]{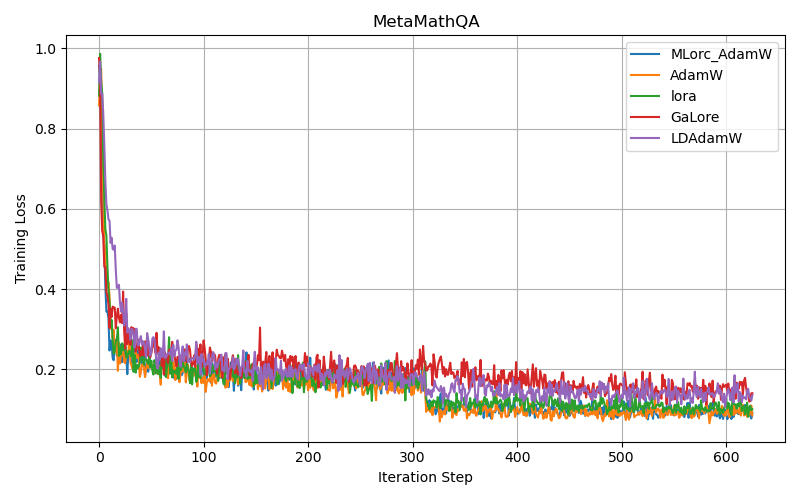}
    \caption{Training Loss on MetaMathQA\citep{yu2023metamath} dataset}
    \label{fig:mathloss_plot_AdamW}
  \end{subfigure}
  \hfill
  \begin{subfigure}[t]{0.48\textwidth}
    \centering
    \includegraphics[width=\linewidth]{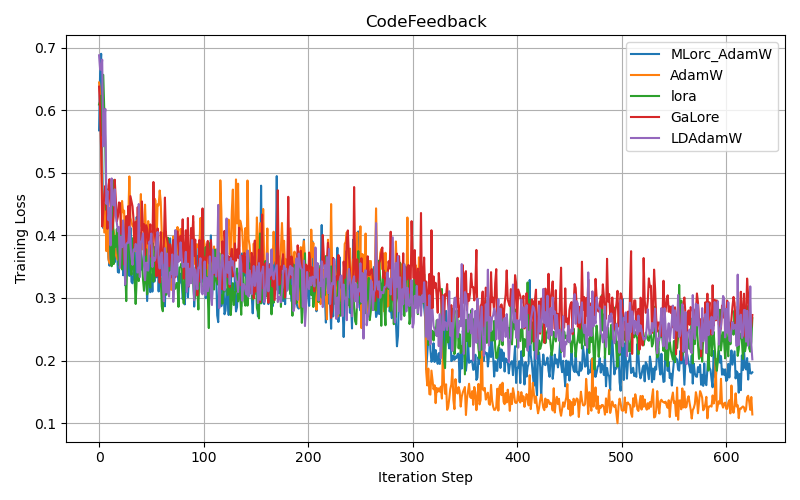}
    \caption{Training Loss on CodeFeedback\citep{zheng2024opencodeinterpreter} dataset}
    \label{fig:codeloss_plot_AdamW}
  \end{subfigure}
  \caption{Training Loss of AdamW of different methods}
  \label{fig:loss_plot_AdamW}
\end{figure*}

As shown in  Table~\ref{tab:results}, MLorc outperforms LoRA \citep{hu2022lora}, GaLore \citep{zhao2024galore} and LDAdamW \citep{robert2024ldadam} on both math and coding task, significantly reduces accuracy gap between Full-parameter fine-tuning and existing memory-efficient training paradigms, demonstrating its ability in handling complex tasks on large language models. Additionally, the optimal learning rate of MLorc is much closer to Full-parameter fine-tuning than other memory-efficient training paradigms (see \Cref{app: setting}), which suggests it might have similar training dynamics and indicates its utility in accelerating convergence.

\begin{figure*}[htbp]
  \centering
  \begin{subfigure}[t]{0.48\textwidth}
    \centering
    \includegraphics[width=\linewidth]{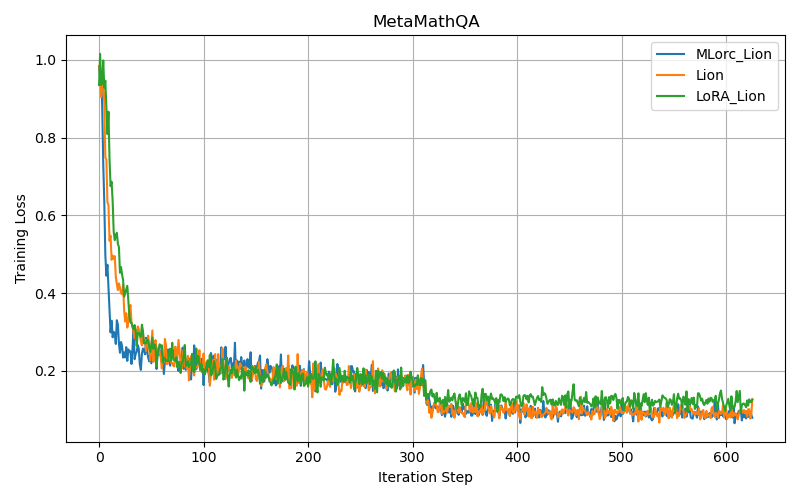}
    \caption{Training Loss on MetaMathQA\citep{yu2023metamath} dataset}
    \label{fig:mathloss_plot_Lion}
  \end{subfigure}
  \hfill
  \begin{subfigure}[t]{0.48\textwidth}
    \centering
    \includegraphics[width=\linewidth]{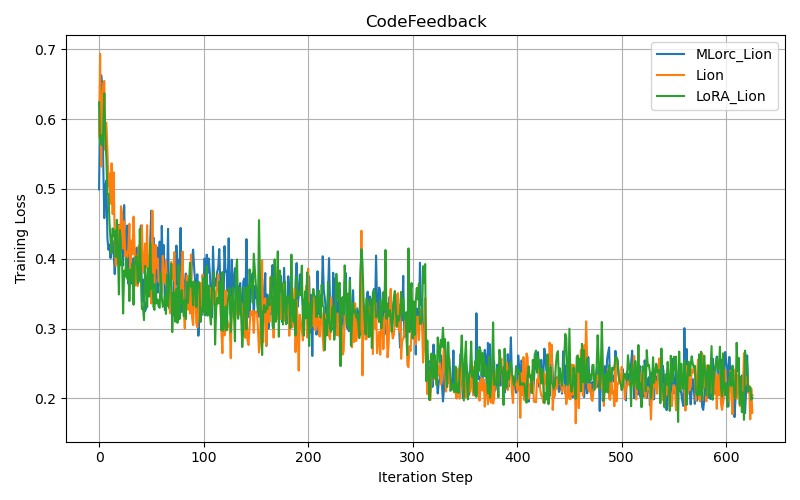}
    \caption{Training Loss on CodeFeedback\citep{zheng2024opencodeinterpreter} dataset}
    \label{fig:codeloss_plot_Lion}
  \end{subfigure}
  \caption{Training Loss of Full Lion and Lion with MLorc}
  \label{fig:loss_plot_Lion}
\end{figure*}

\textbf{Training Loss Curve.}  As shown in Figure~\ref{fig:loss_plot_AdamW} and~\ref{fig:loss_plot_Lion}, in most cases, the training loss of MLorc is  smaller than other memory-efficient training paradigms, and is usually close to the full version, whenever in AdamW or Lion. This shows that MLorc behaves similarly to the full version, hence providing strong evidence for its effectiveness.

\begin{table}[htbp]
\centering
\caption{Memory consumption of different methods with AdamW optimizer when training on MetaMathQA\citep{yu2023metamath}. Hyperparameters and other settings are same as previous experiments.}
\label{tab:memory}
\begin{tabular}{|c|c|c|c|}
\hline
MLorc & LoRA & GaLore & LDAdamW\\
\hline
44.8GB & 45.6GB & 44.8GB & 54.6GB\\
\hline
\end{tabular}
\end{table}

\begin{table}[htbp]
\centering
\caption{Training time of different memory-efficient training methods with AdamW optimizer when training on MetaMathQA \citep{yu2023metamath}. Hyperparameters and other settings are the same as the previous experiments.}
\label{tab:time}
\begin{tabular}{|c|c|c|c|}
\hline
MLorc & LoRA & GaLore & LDAdamW \\
\hline
1h25min & 1h24min & 1h33min & 1h26min \\
\hline
\end{tabular}
\end{table}

\textbf{Time and Memory Efficiency.} Table~\ref{tab:memory} compares memory consumption between different training methods. MLorc, GaLore and LoRA have similar memory efficiency; LDAdamW consumes more memory, probably due to its error feedback mechanism. Table~\ref{tab:time} compares training time between different memory-efficient training methods. Experimental result shows that MLorc achieves time efficiency comparable to that of LoRA \citep{hu2022lora} and reduces training time compared to GaLore \citep{zhao2024galore}, confirming that the additional overhead from compression and reconstruction is negligible in practical fine-tuning scenarios.

\begin{table*}[h]
\centering
\caption{GLUE benchmark results of memory-efficient fine-tuning methods using pre-trained RoBERTa-Base. We set rank as 8 for all four memory-efficient methods. We used AdamW as the optimizer in full finetuning and LoRA. MLorc and Galore refer to MLorc-AdamW and Galore-AdamW respectively. Best performances among MLorc, Lora, GaLore and LDAdamW are highlighted in bold.}
\label{tab:glue_results}
\begin{tabular}{lccccccccc}
\toprule
Method & CoLA & MNLI & MRPC & QNLI & QQP & RTE & SST2 & STSB & Avg \\
\midrule
Full     & 62.33 & 87.62 & 91.11 & 92.92 & 90.26 & 75.81 & 95.18 & 90.50 & 85.72 \\ \midrule
MLorc    & \textbf{62.07} & \textbf{87.53} & \textbf{90.77} & \textbf{93.19} & 88.99 & \textbf{77.98} & \textbf{95.18} & 90.59 & \textbf{85.79} \\
LoRA     & 61.53 & 87.51 & 90.10 & 92.75 & 89.45 & 76.53 & 94.72 & \textbf{90.74} & 85.42 \\
GaLore   & 60.34 & 86.84 & 90.10 & 92.42 & 88.15 & 71.12 & 94.38 & 90.50 & 84.23 \\
LDAdamW  & 60.82 & 87.33 & 90.35 & 93.03 & \textbf{90.06} & 76.53 & 94.61 & \textbf{90.74} & 85.43 \\
\bottomrule
\end{tabular}
\end{table*}

\subsection{Experiments on Natural Language Understanding Tasks}

In this subsection, we assess the effectiveness of MLorc in fine-tuning language models for natural language understanding (NLU) tasks. Specifically, we fine-tune pre-trained RoBERTa models \citep{liu2019roberta} on the GLUE benchmark \citep{wang2018glue} using MLorc-AdamW, and compare its performance against full fine-tuning, LoRA \citep{hu2022lora}, GaLore \citep{zhao2024galore} and LDAdamW \citep{robert2024ldadam}. As shown in \Cref{tab:glue_results}, MLorc significantly outperforms GaLore, surpasses LoRA and LDAdamW on most tasks, and achieves overall performance comparable to that of full fine-tuning. Detailed experimental settings can be found in \Cref{app: nlu}.

We also performed experiments to examine the low-rank structure of the gradient, first-order momentum, and second-order momentum during the full fine-tuning process of AdamW. The results on the STSB dataset are shown in Figure \ref{fig: moment stsb}, and additional experimental results can be found in \Cref{app: low rank}.

\section{CONCLUSIONS AND FUTURE WORK}\label{section: con}

In this work, we introduce MLorc (Momentum Low-rank Compression), a novel memory-efficient optimization paradigm designed to bridge the gap between parameter-efficient fine-tuning and full-parameter training for large language models (LLMs). Unlike existing low-rank adaptation methods such as LoRA \citep{hu2022lora} and GaLore \citep{zhao2024galore}, MLorc leverages a previously underexplored insight: momentum is low-rank, and it can be compressed without significantly breaking training dynamics. By applying Randomized SVD (RSVD) \citep{halko2011finding} to compress and reconstruct momentum states instead of gradients, MLorc achieves a better balance between memory savings and training dynamics preservation.

Through comprehensive empirical evaluations across various model architectures, optimizers (e.g., AdamW and Lion), and NLP tasks (including NLG and NLU benchmarks), we demonstrate that MLorc: (1) consistently outperforms LoRA, GaLore and LDAdamW in terms of validation accuracy; (2) matches or exceeds full fine-tuning performance with a small compression rank (e.g., \(r=4\)); (3) maintains comparable memory consumption to LoRA and achieves better time efficiency than GaLore; (4) shows strong training dynamics alignment with full fine-tuning, as evidenced by its training loss curves and optimal learning rates.

MLorc contributes to reducing the environmental impact of large-scale model training by significantly lowering GPU memory usage and computation overhead, which can lead to reduced energy consumption during fine-tuning.

Looking ahead, we identify several promising directions for future work to validate and enhance MLorc: (1) Although our experiments focus on fine-tuning, extending MLorc to large-scale pre-training holds strong potential, as demonstrated by GaLore \citep{zhao2024galore}, which shows that memory-efficient training schemes can significantly reduce memory usage while preserving model quality;  (2) the current approach to compress momentum may not be optimal, so exploring alternative compression strategies with improved memory/time efficiency and approximation accuracy could be valuable; (3) our experiments were limited to models up to 7B parameters, so further empirical evaluation on larger-scale models (e.g., GPT-3 \citep{brown2020language}) would help assess the scalability and effectiveness of MLorc.

\section*{Acknowledgments}
The work of Jiawei Zhang was supported by the Office of the Vice Chancellor for Research and Graduate Education at the University of Wisconsin–Madison with funding from the Wisconsin Alumni Research Foundation.

\bibliographystyle{apalike}
\bibliography{ref}

\section*{Checklist}

The checklist follows the references. For each question, choose your answer from the three possible options: Yes, No, Not Applicable.  You are encouraged to include a justification to your answer, either by referencing the appropriate section of your paper or providing a brief inline description (1-2 sentences). 
Please do not modify the questions.  Note that the Checklist section does not count towards the page limit. Not including the checklist in the first submission won't result in desk rejection, although in such case we will ask you to upload it during the author response period and include it in camera ready (if accepted).

\textbf{In your paper, please delete this instructions block and only keep the Checklist section heading above along with the questions/answers below.}

\begin{enumerate}

  \item For all models and algorithms presented, check if you include:
  \begin{enumerate}
    \item A clear description of the mathematical setting, assumptions, algorithm, and/or model. [Yes] We clearly describe the assumptions in \Cref{ass: smooth}, \ref{ass: variance}; algorithms in \Cref{alg: adam}, \ref{alg: lion}.
    \item An analysis of the properties and complexity (time, space, sample size) of any algorithm. [Yes] We analysis the memory consumption in \Cref{section: memory consumption}, and provide empirical memory consumption and training time of algorithms in \Cref{tab:memory} and \ref{tab:time}.
    \item (Optional) Anonymized source code, with specification of all dependencies, including external libraries. [Yes] We provide anonymized source code in the Supplementary Material.
  \end{enumerate}

  \item For any theoretical claim, check if you include:
  \begin{enumerate}
    \item Statements of the full set of assumptions of all theoretical results. [Yes] We clearly state the assumptions in \Cref{ass: smooth}, \ref{ass: variance}.
    \item Complete proofs of all theoretical results. [Yes] We provide complete proofs in \Cref{app: proof}.
    \item Clear explanations of any assumptions. [Yes] We explain the   assumptions in \Cref{section: convergence}.  
  \end{enumerate}

  \item For all figures and tables that present empirical results, check if you include:
  \begin{enumerate}
    \item The code, data, and instructions needed to reproduce the main experimental results (either in the supplemental material or as a URL). [Yes] We provide the code and instructions needed to reproduce the main experimental results in the Supplementary Material.
    \item All the training details (e.g., data splits, hyperparameters, how they were chosen). [Yes] We mentioned those details in \Cref{section: exp} and \Cref{app: exp}.
    \item A clear definition of the specific measure or statistics and error bars (e.g., with respect to the random seed after running experiments multiple times). [Yes] We mentioned the details of error bars in \Cref{section: exp}: it's the standard deviation of four evaluation.
    \item A description of the computing infrastructure used. (e.g., type of GPUs, internal cluster, or cloud provider). [Yes] We mentioned the computing infrastructure used in \Cref{section: exp} and \Cref{app: exp}.
  \end{enumerate}

  \item If you are using existing assets (e.g., code, data, models) or curating/releasing new assets, check if you include:
  \begin{enumerate}
    \item Citations of the creator If your work uses existing assets. [Yes] We cited the existing assets we used.
    \item The license information of the assets, if applicable. [Yes] We provide the license information of the assets in \Cref{app: lincese}.
    \item New assets either in the supplemental material or as a URL, if applicable. [Yes] We upload our source code in the Supplementary Material.
    \item Information about consent from data providers/curators. [Not Applicable]
    \item Discussion of sensible content if applicable, e.g., personally identifiable information or offensive content. [Not Applicable]
  \end{enumerate}

  \item If you used crowdsourcing or conducted research with human subjects, check if you include:
  \begin{enumerate}
    \item The full text of instructions given to participants and screenshots. [Not Applicable]
    \item Descriptions of potential participant risks, with links to Institutional Review Board (IRB) approvals if applicable. [Not Applicable]
    \item The estimated hourly wage paid to participants and the total amount spent on participant compensation. [Not Applicable]
  \end{enumerate}

\end{enumerate}

\clearpage
\appendix
\thispagestyle{empty}

\onecolumn

\newpage
\section*{Appendix}

Appendix is organized as follows. \Cref{app: rsvd} introduces details on RSVD. \Cref{app: proof} provides a complete proof of \Cref{thm: lion}. \Cref{app: exp} presents additional experimental evidence on the low-rank structure and memory efficiency of MLorc. \Cref{app: setting} gives detailed hyperparameter settings of experiments in \Cref{section: exp} for reproducibility.

\section{Details on RSVD}\label{app: rsvd}
Randomized Singular Value Decomposition (RSVD) is an efficient algorithm for computing a low-rank approximation of large matrices. Unlike the classical SVD, which can be computationally expensive for large-scale data, RSVD uses random projections to reduce the dimensionality of the input matrix before performing decomposition. This significantly accelerates the computation while retaining high approximation accuracy.\\

\begin{algorithm}
\caption{Randomized SVD (RSVD) with Oversampling}
\begin{algorithmic}
\Require Matrix \( A \in \mathbb{R}^{m \times n} \), target rank \( r \), oversampling parameter \( p \)
\Ensure Approximate rank-\( r \) SVD: \( A \approx U \Sigma V^\top \)
\State \( l \gets r + p \) 
\State Generate a random Gaussian matrix \( \Omega \in \mathbb{R}^{n \times l} \)
\State \( Y \gets A \Omega \in \mathbb{R}^{m \times l} \)
\State Compute the QR decomposition: \( Y = QR \)
\State \( B \gets Q^\top A \in \mathbb{R}^{l \times n} \)
\State Compute SVD of the small matrix: \( \tilde{U}, \Sigma, V^\top = SVD(B) \)
\State \( U \gets Q \tilde{U} \)
\State \Return \( U, \Sigma, V \)
\end{algorithmic}
\end{algorithm}
The key idea behind RSVD is to first project the original matrix $A \in \mathbb{R}^{m \times n}$ onto a lower-dimensional subspace using a random matrix \(\Omega\) producing a smaller matrix \(Y = A\Omega.\) Then, it performs a standard SVD or QR decomposition on \(Y\), and reconstructs the approximate SVD \(A\) from this compressed representation.

Concerning the precision of RSVD, we have the following theorem, which is Theorem 10.5 of \citep{halko2011finding}:

\begin{lemma}[Approximation error bound of RSVD]\label{lemma: rsvd}
    Let \( A \in \mathbb{R}^{m \times n} \) have singular values \( \sigma_1 \geq \sigma_2 \geq \cdots \). For a target rank \( r \geq 2\) and oversampling parameter \( p \geq 2\). When $r+p\leq \min\{m,n\}$, the randomized SVD algorithm produces an approximation \( A_{RS} \) such that
\begin{equation}
\mathbb{E} \left[ \| A - A_{RS} \|_F \right] \leq \left(1 + \frac{r}{p - 1} \right)^{\frac{1}{2}} (\sum_{j > r} \sigma_j^2)^{\frac{1}{2}}.
\label{eq:rsvd_error}
\end{equation}
\end{lemma}

This lemma suggests that when the desired rank \(r\) is small (e.g, \(r=4\)), with a proper oversampling parameter \( p \), RSVD has the same level approximation error(in expectation) as exact SVD up to a constant. We note that though the oversampling parameter \( p \) is involved in providing an error bound for RSVD, empirically, it does not significantly influence the experimental result. To reduce computational overhead, we set oversampling parameter \( p=0 \) in each of our experiments.

\section{Proofs}\label{app: proof}
Our proof generally follows that of \citep{dong2024convergence}, while additionally analyzing and bounding the error introduced by the low-rank compression in MLorc.

\subsection{Lemmas}

\begin{lemma}\label{lemma: compression error}
When $r\geq 2, p\geq 2$ and $r+p\leq \min\{m,n\}$, define $\gamma=\left(1 + \frac{r}{p - 1} \right)^{\frac{1}{2}}$. Then we have
    \begin{align*}
        \E[\|\tilde{m}_t- m_t\|_F]\leq \gamma\|g_t\|_F
    \end{align*}
\end{lemma}
\begin{proof}
Let  \( \sigma_1 \geq \sigma_2 \geq \cdots \)  be the singular values of $m_t$. We have
    \begin{align*}
        \E[\|\tilde{m}_t- m_t\|_F]\leq& \left(1 + \frac{r}{p - 1} \right)^{\frac{1}{2}}(\sum_{j > r} \sigma_j^2)^{\frac{1}{2}}\\
        \leq &\left(1 + \frac{r}{p - 1} \right)^{\frac{1}{2}}\|m_t-\beta_2 \tilde{m}_{t-1}\|_F\\
        =&\left(1 + \frac{r}{p - 1} \right)^{\frac{1}{2}}(1-\beta_2)\|g_t\|_F,
    \end{align*}
where the first inequality is due to \Cref{lemma: rsvd}, while the second inequality follows from the Eckart–Young–Mirsky theorem, noting that $\tilde{m}_{t-1}$ is a rank $r$ matrix.
\end{proof}

\begin{lemma}\label{lemma: momentum error}
Denote $\delta_t= c_t-\nabla f( W_t)$.  Under \Cref{ass: smooth}, \ref{ass: variance}, we have
\begin{align}
\begin{aligned}\notag
\frac{1}{T}\sum_{t=1}^T\E\left[\|\delta_t\|_{1,1}\right]\leq& \frac{\sqrt{d}\sigma}{\sqrt{b} T(1-\beta_2)} + \frac{2L\alpha d}{1-\beta_2}+ \left(|\beta_1-\beta_2|+(1-\beta_1)\right)\cdot\frac{\sqrt{d}\sigma}{\sqrt{b(1-\beta_2)}}\\
&+\frac{1}{T}\sum_{t=1}^{T-1} \beta_1\sqrt{d}\gamma\E\left[\|\nabla f( W_t)\|_{1,1}\right]+ \gamma\frac{\sigma\sqrt{d}}{\sqrt{b}}
\end{aligned}
\end{align}
\end{lemma}
\begin{proof}
Denote  $\xi_t= g_t-\nabla f( W_t)$. We have
\begin{align*}
\delta_t=&\beta_1\tilde{m}_{t-1}+(1-\beta_1) g_t-\nabla f( W_t)\\
=&\beta_1 (\tilde{m}_{t-1}- m_{t-1})+ \beta_1\beta_2\tilde{m}_{t-2}+\beta_1(1-\beta_2)\nabla f(W_{t-1})+(1-\beta_1) g_t-\nabla f( W_t)\\
=&\beta_1 (\tilde{m}_{t-1}- m_{t-1})+ \beta_2( c_{t-1}-(1-\beta_1)\nabla f(W_{t-1}))+\beta_1(1-\beta_2)\nabla f(W_{t-1})+(1-\beta_1) g_t-\nabla f( W_t)\\
=&\beta_1 (\tilde{m}_{t-1}- m_{t-1})+ \beta_2(\delta_{t-1}+\nabla f( W_{t-1}))-\beta_2(1-\beta_1)g_{t-1}+\beta_1(1-\beta_2)\nabla f(W_{t-1})+(1-\beta_1) g_t-\nabla f( W_t)\\
=&\beta_1 (\tilde{m}_{t-1}- m_{t-1})+ \beta_2(\delta_{t-1}+\nabla f( W_{t-1}))+(\beta_1-\beta_2)(\xi_{t-1}+\nabla f( W_{t-1}))+(1-\beta_1)(\xi_t+\nabla f( W_t))-\nabla f( W_t)\\
=&\beta_1 (\tilde{m}_{t-1}- m_{t-1})+ \beta_2\delta_{t-1}-\beta_1(\nabla f( W_t)-\nabla f( W_{t-1}))+(\beta_1-\beta_2)\xi_{t-1}+(1-\beta_1)\xi_t\\
=&\beta_2^{t-1}\delta_1 + \sum_{k=2}^k\beta_2^{t-k}\Bigg(-\beta_1\left(\nabla f( W_k)-\nabla f( W_{k-1})\right)+(\beta_1-\beta_2)\xi_{k-1}+(1-\beta_1)\xi_k+\beta_1(\tilde{m}_{k-1}- m_{k-1})\Bigg)\\
=&\beta_2^{t-1}\delta_1 - \beta_1\sum_{k=2}^k\beta_2^{t-k}\left(\nabla f( W_k)-\nabla f( W_{k-1})\right)+(\beta_1-\beta_2)\sum_{k=2}^t\beta_2^{t-k}\xi_{k-1}\\
&+ (1-\beta_1)\sum_{k=2}^t\beta_2^{t-k}\xi_k + \beta_1\sum_{k=2}^t\beta_2^{t-k}(\tilde{m}_{k-1}- m_{k-1}).
\end{align*}

Taking expectations, and according to \Cref{lemma: momentum error}, we have
\begin{align}
\begin{aligned}\notag
\E\left[\|\delta_t\|_{1,1}\right]\leq&\sqrt{d}\Bigg\{ \beta_2^{t-1}\E\left[\|\delta_1\|_F\right] + \beta_1\underbrace{\sum_{k=2}^t\beta_2^{t-k}\E\left[\left\|\nabla f( W_k)-\nabla f( W_{k-1})\right\|_F\right]}_{\text{\rm term (a)}}\\
&+\underbrace{\E\left[\left\| (\beta_1-\beta_2)\sum_{k=2}^t\beta_2^{t-k}\xi_{k-1}+ (1-\beta_1)\sum_{k=2}^t\beta_2^{t-k}\xi_k\right\|_F\right]}_{\text{\rm term (b)}} \\
&+\beta_1\underbrace{\sum_{k=2}^t\beta_2^{t-k}\E\left[\left\|\tilde{m}_{k-1}- m_{k-1}\right\|_F\right]}_{\text{\rm term (c)}}\Bigg\}.
\end{aligned}
\end{align}

For term (a), we have 
\begin{align}
\begin{aligned}\notag
\mbox{term (a)}\leq& L\sum_{k=2}^t\beta_2^{t-k}\E\left[\| W_k- W_{k-1}\|_F\right]\\
=& L\alpha\sum_{k=2}^t\beta_2^{t-k}\E\left[\|\sign(c^{t-1})\right]\\
\leq& 2L\alpha\sqrt{d}\sum_{k=2}^t\beta_2^{t-k}\\
\leq& \frac{2L\alpha\sqrt{d}}{1-\beta_2}.
\end{aligned}
\end{align}
For term (b), we have
\begin{equation*}
\begin{aligned}
\text{term (b)} &\leq |\beta_1-\beta_2|\E\left[\left\|\sum_{k=2}^t\beta_2^{t-k}\xi_{k-1}\right\|_F\right]+(1-\beta_1)\E\left[\left\|\sum_{k=2}^t\beta_2^{t-k}\xi_k\right\|_F\right] \\
&\leq |\beta_1-\beta_2|\sqrt{\E\left[\left\|\sum_{k=2}^t\beta_2^{t-k}\xi_{k-1}\right\|_F^2\right]}+(1-\beta_1)\sqrt{\E\left[\left\|\sum_{k=2}^t\beta_2^{t-k}\xi_k\right\|_F^2\right]}\\
&=|\beta_1-\beta_2|\sqrt{\sum_{k=2}^t\beta_2^{2(t-k)}\E\left[\left\|\xi_{k-1}\right\|_F^2\right]}+(1-\beta_1)\sqrt{\sum_{k=2}^t\beta_2^{2(t-k)}\E\left[\left\|\xi_k\right\|_F^2\right]}\\
&=|\beta_1-\beta_2|\sqrt{\sigma^2\sum_{k=2}^t\beta_2^{2(t-k)}/b}+(1-\beta_1)\sqrt{\sigma^2\sum_{k=2}^t\beta_2^{2(t-k)}/b}\\
&\leq\left(|\beta_1-\beta_2|+(1-\beta_1)\right)\cdot\frac{\sigma}{\sqrt{b(1-\beta_2^2)}}\\
&\leq\left(|\beta_1-\beta_2|+(1-\beta_1)\right)\cdot\frac{\sigma}{\sqrt{b(1-\beta_2)}}.
\end{aligned}
\end{equation*}
For term (c), according to \Cref{lemma: compression error}, we have 
\begin{align}
\begin{aligned}\notag
\mbox{term (c)}\leq& \sum_{k=2}^t\beta_2^{t-k}\E\left[\|\tilde{m}_{k-1}- m_{k-1}\|_F\right]\\
\leq& \gamma(1-\beta_2)\sum_{k=2}^t\beta_2^{t-k}\E\left[\|g_{k-1}\|_F\right]\\
\leq& \gamma(1-\beta_2)\sum_{k=2}^t\beta_2^{t-k}\E\left[\|\nabla f( W_{k-1})\|_F\right]+\gamma(1-\beta_2)\sum_{k=2}^t\beta_2^{t-k}\E\left[\|\xi_{k-1}\|_F\right]\\
\leq& \gamma(1-\beta_2)\sum_{k=2}^t\beta_2^{t-k}\E\left[\|\nabla f( W_{k-1})\|_{1,1}\right]+ \gamma\frac{\sigma}{\sqrt{b}}.
\end{aligned}
\end{align}

Plugging terms (a), (b), (c) back, we get
\begin{align}
\begin{aligned}\notag
\E\left[\|\delta_t\|_{1,1}\right]\leq&\sqrt{d}\{\beta_2^{k-1}\E\left[\|\delta_1\|_F\right] + \frac{2\beta_1 L\alpha\sqrt{d}}{1-\beta_2} + \left(|\beta_1-\beta_2|+(1-\beta_1)\right)\cdot\frac{\sigma}{\sqrt{b(1-\beta_2)}}\}\\
&+\sqrt{d}\gamma(1-\beta_2)\beta_1\sum_{k=2}^t\beta_2^{t-k}\E\left[\|\nabla f( W_{k-1})\|_{1,1}\right]+ \gamma\frac{\sigma\sqrt{d}}{\sqrt{b}}\\
\end{aligned}
\end{align}
Initializing $ m_0= g_1$, we have $\E\left[\|\delta_1\|_F\right]=\E\left[\| g_1-\nabla f( W_{1})\|_F\right]\leq\sigma/\sqrt{b}$, and
\begin{align}
\begin{aligned}\notag
\frac{1}{T}\sum_{t=1}^T\E\left[\|\delta_t\|_{1,1}\right]\leq& \frac{\sqrt{d}\sigma}{\sqrt{b} T(1-\beta_2)} + \frac{2L\alpha d}{1-\beta_2}+ \left(|\beta_1-\beta_2|+(1-\beta_1)\right)\cdot\frac{\sqrt{d}\sigma}{\sqrt{b(1-\beta_2)}}\\
&+\frac{1}{T}\sum_{t=1}^{T-1} \beta_1\sqrt{d}\gamma\E\left[\|\nabla f( W_t)\|_{1,1}\right]+ \gamma\frac{\sigma\sqrt{d}}{\sqrt{b}}
\end{aligned}
\end{align}

\end{proof}

\subsection{Proof of \Cref{thm: lion}}

\textbf{Formal statement of \Cref{thm: lion}}:     
Under Assumptions \ref{ass: smooth} and \ref{ass: variance}, applying \Cref{alg: lion} with $r\geq 2, p\geq 2$, $r+p\leq \min\{m,n\}$, $\beta_1\leq \frac{1}{4\gamma\sqrt{d}}$, we have
\begin{align}
\begin{aligned}\notag
\frac{1}{T}\sum_{t=1}^T&\E\left[\left\|\nabla f( W_t)\right\|_{1,1}\right]\leq O(1)\Bigg[\frac{f( W_{1})-f(W_T)}{\alpha T}+\frac{\sigma\sqrt{d}}{\sqrt{b}T(1-\beta_2)}\\
&+ \frac{2L\alpha \beta_1 d}{1-\beta_2}+ d L \alpha+ \gamma\frac{\sigma\sqrt{d}}{\sqrt{b}}+(|\beta_1-\beta_2|+1-\beta_1) \cdot\frac{\sqrt{d}\sigma}{\sqrt{b(1-\beta_2)}}\Bigg].
\end{aligned}
\end{align}
Denote $f( W_{1})-\inf_W f(W)=\Delta$. Set $\gamma=\left(1 + \frac{r}{p - 1} \right)^{\frac{1}{2}}=O(1)$, $\beta_2=O(1)$, $\alpha=\sqrt{\frac{\Delta}{LdT}}$. We have
    \begin{align*}
        \frac{1}{T}\sum_{t=1}^{T}\E[\|\nabla f(W_t)\|_{1,1}]\leq O(1)\left[\frac{\sqrt{dL\Delta}}{\sqrt{T}}+\frac{ \sigma \sqrt{d}}{\sqrt{b}}\right].
    \end{align*}

\begin{proof}
    
According to \Cref{ass: smooth}, we have
\begin{align}
\begin{aligned}\notag
f&( W_{t+1})-f( W_t)\\\leq&\<\nabla f( W_t), W_{t+1}- W_t\>+\frac{L}{2}\| W_{t+1}- W_t\|_F^2\\
=&-\alpha\<\nabla f( W_t),\sign( c_t)\>+\frac{L\alpha^2}{2}\|\sign( c_t)\|_F^2\\
=&-\alpha\<\nabla f( W_t),\sign(\nabla f( W_t))\>-\alpha\<\nabla f( W_t),\sign( c_t)-\sign(\nabla f( W_t))\>+\frac{L\alpha^2}{2}\|\sign( c_t)\|_F^2\\
\leq&-\alpha\left\|\nabla f( W_t)\right\|_{1,1}+2\alpha\|\delta_t\|_{1,1}+\frac{dL\alpha^2}{2}
\end{aligned}
\end{align}

Taking expectations, and according to \Cref{lemma: momentum error}, we have
\begin{align}
\begin{aligned}\notag
\E&\left[f( W_{T})-f( W_{0})\right]\\
\leq&-\alpha\sum_{t=0}^{T-1}\E[\|\nabla f( W_t)\|_{1,1}]+2\alpha\sum_{k=1}^T\E\left[\|\delta_t\|_{1,1}\right]+TdL\alpha^2/2\\
\leq&-\alpha\sum_{t=1}^T\E[\|\nabla f( W_t)\|_{1,1}]+2\alpha T\Bigg\{\frac{\sqrt{d}\sigma}{\sqrt{b} T(1-\beta_2)} + \frac{2L\alpha \beta_1 d}{1-\beta_2} + (|\beta_1-\beta_2|+1-\beta_1) \cdot\frac{\sqrt{d}\sigma}{\sqrt{b(1-\beta_2)}}\\
&+\frac{1}{T}\sum_{t=1}^{T-1} \beta_1\sqrt{d}\gamma\E\left[\|\nabla f( W_t)\|_{1,1}\right]+ \gamma\frac{\sigma\sqrt{d}}{\sqrt{b}}\Bigg\}+TdL\alpha^2/2.
\end{aligned}
\end{align}

When $\beta_1\leq \frac{1}{4\gamma\sqrt{d}}$, we have
\begin{align*}
    1-2\beta_1\sqrt{d}\gamma\geq \frac{1}{2},
\end{align*}
and
\begin{align}
\begin{aligned}\notag
\frac{1}{T}\sum_{t=0}^{T-1}&\E\left[\left\|\nabla f( W_t)\right\|_{1,1}\right]\leq O(1)\Bigg[\frac{f( W_{1})-f(W_T)}{\alpha T}+\frac{\sigma\sqrt{d}}{\sqrt{b}T(1-\beta_2)}\\
&+ \frac{2L\alpha \beta_1 d}{1-\beta_2}+ d L \alpha+ \gamma\frac{\sigma\sqrt{d}}{\sqrt{b}}+(|\beta_1-\beta_2|+1-\beta_1) \cdot\frac{\sqrt{d}\sigma}{\sqrt{b(1-\beta_2)}}\Bigg],
\end{aligned}
\end{align}
Denote $f( W_{1})-\inf_W f(W)=\Delta$. Set $\gamma=\left(1 + \frac{r}{p - 1} \right)^{\frac{1}{2}}=O(1)$, $\beta_2=O(1)$, $\alpha=\sqrt{\frac{\Delta}{LdT}}$. We have
    \begin{align*}
        \frac{1}{T}\sum_{t=1}^{T}\E[\|\nabla f(W_t)\|_{1,1}]\leq O(1)\left[\frac{\sqrt{dL\Delta}}{\sqrt{T}}+\frac{ \sigma \sqrt{d}}{\sqrt{b}}\right].
    \end{align*}
Thus, when $\sigma=0$ in deterministic case, we can find an $\epsilon$-entrywise $\ell_1$-norm stationary point of $f$ with a complexity of $O(\Delta L d \epsilon^{-2})$; when $\sigma\neq 0$ in stochastic case, with a large batch size $b = \Theta(d\sigma^2\epsilon^{-2})$, we can find an $\epsilon$-entrywise $\ell_1$-norm stationary point of $f$ with a sample complexity of $O(\Delta L d^2 \sigma^2 \epsilon^{-4})$, matching the same sample complexity as the original Lion \citep{dong2024convergence}.

\end{proof}

\section{Additional Experimental Results}\label{app: exp}
\subsection{Low-rank Structures of the Gradients and Momenta}\label{app: low rank}

\begin{figure}[htbp]
    \centering

    \begin{subfigure}[b]{0.4\textwidth}
        \centering
        \includegraphics[width=\textwidth]{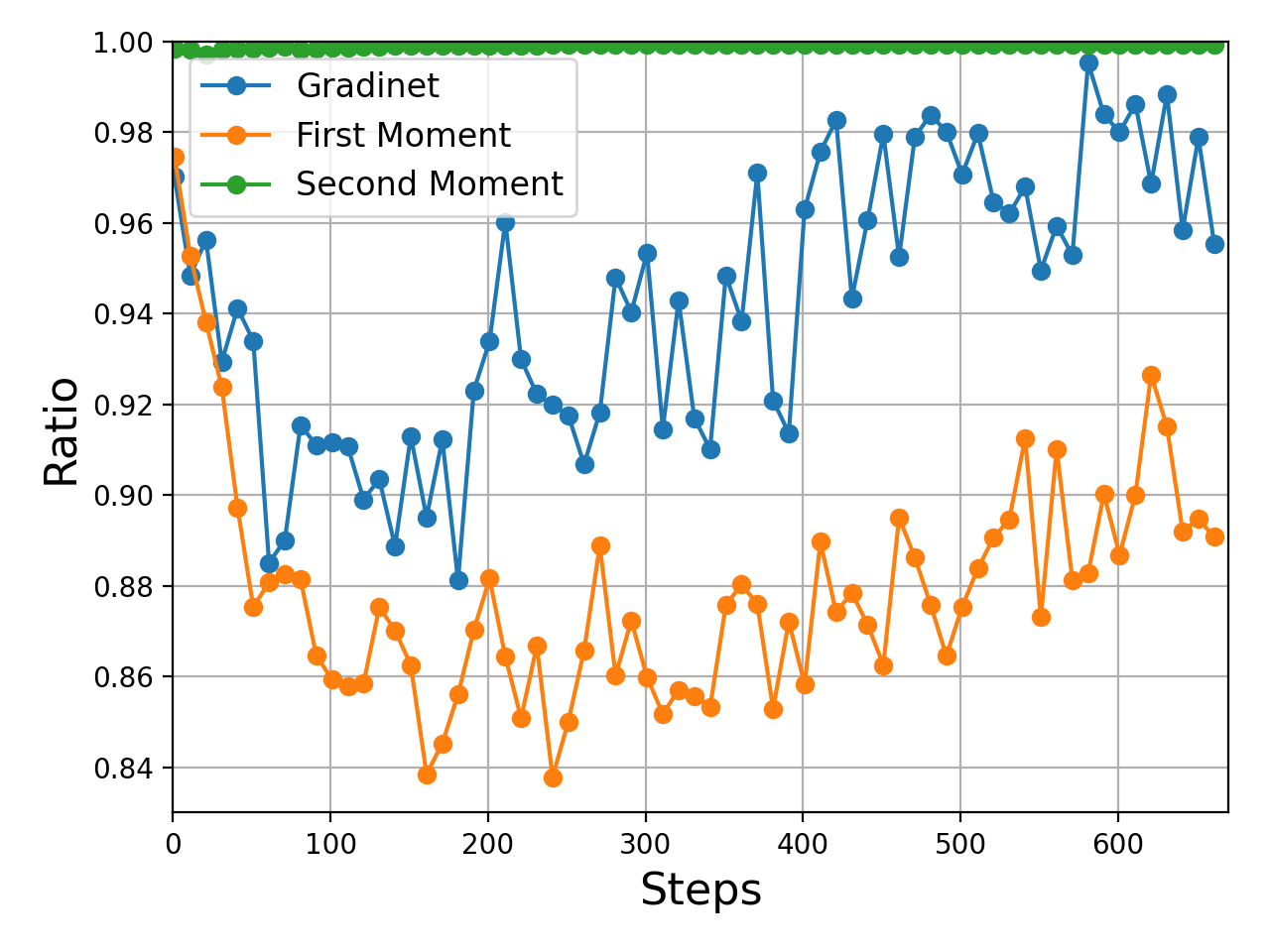}
        \caption{CoLA}
    \end{subfigure}
    \begin{subfigure}[b]{0.4\textwidth}
        \centering
        \includegraphics[width=\textwidth]{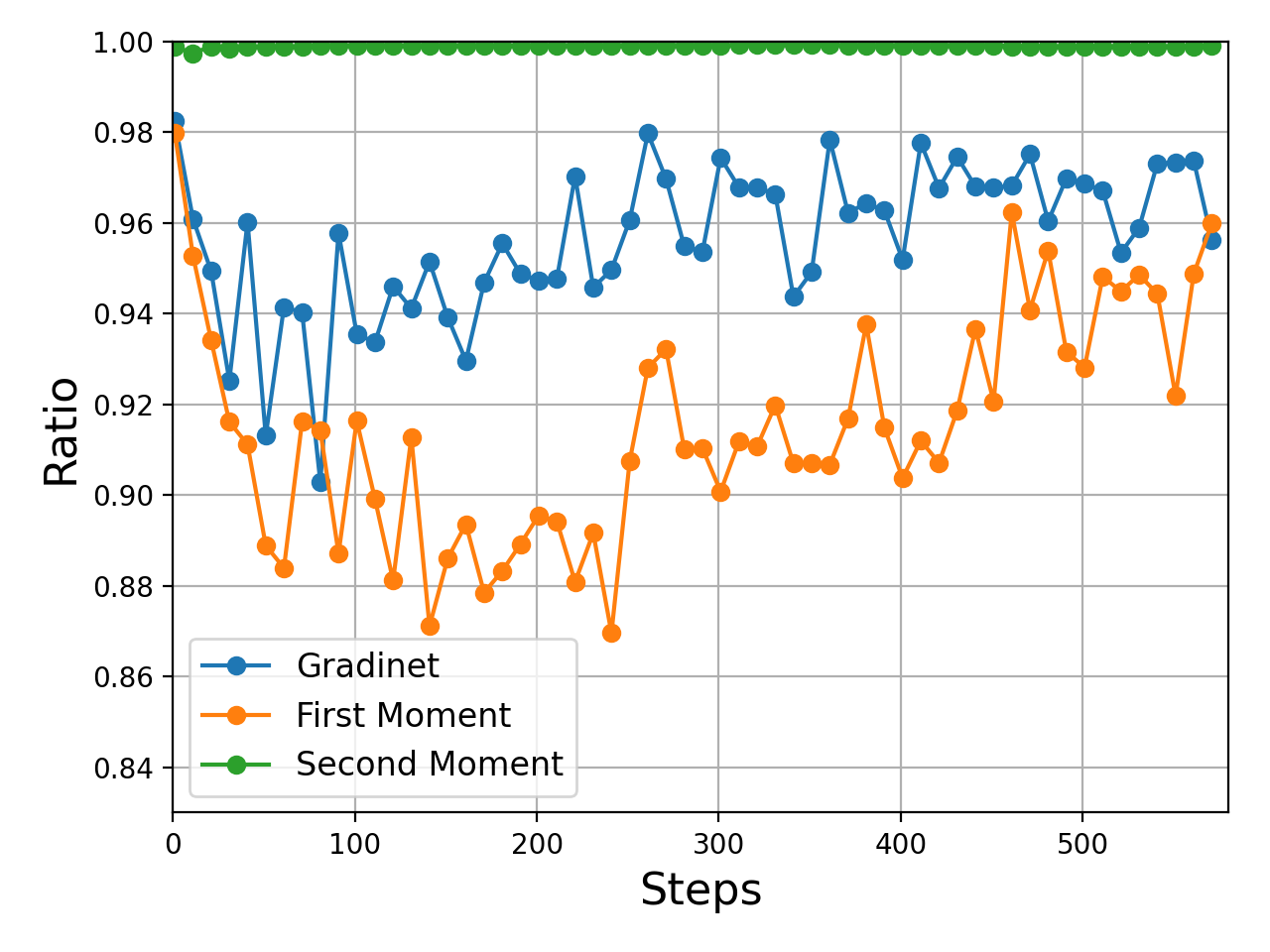}
        \caption{MRPC}
    \end{subfigure}


    \begin{subfigure}[b]{0.4\textwidth}
        \centering
        \includegraphics[width=\textwidth]{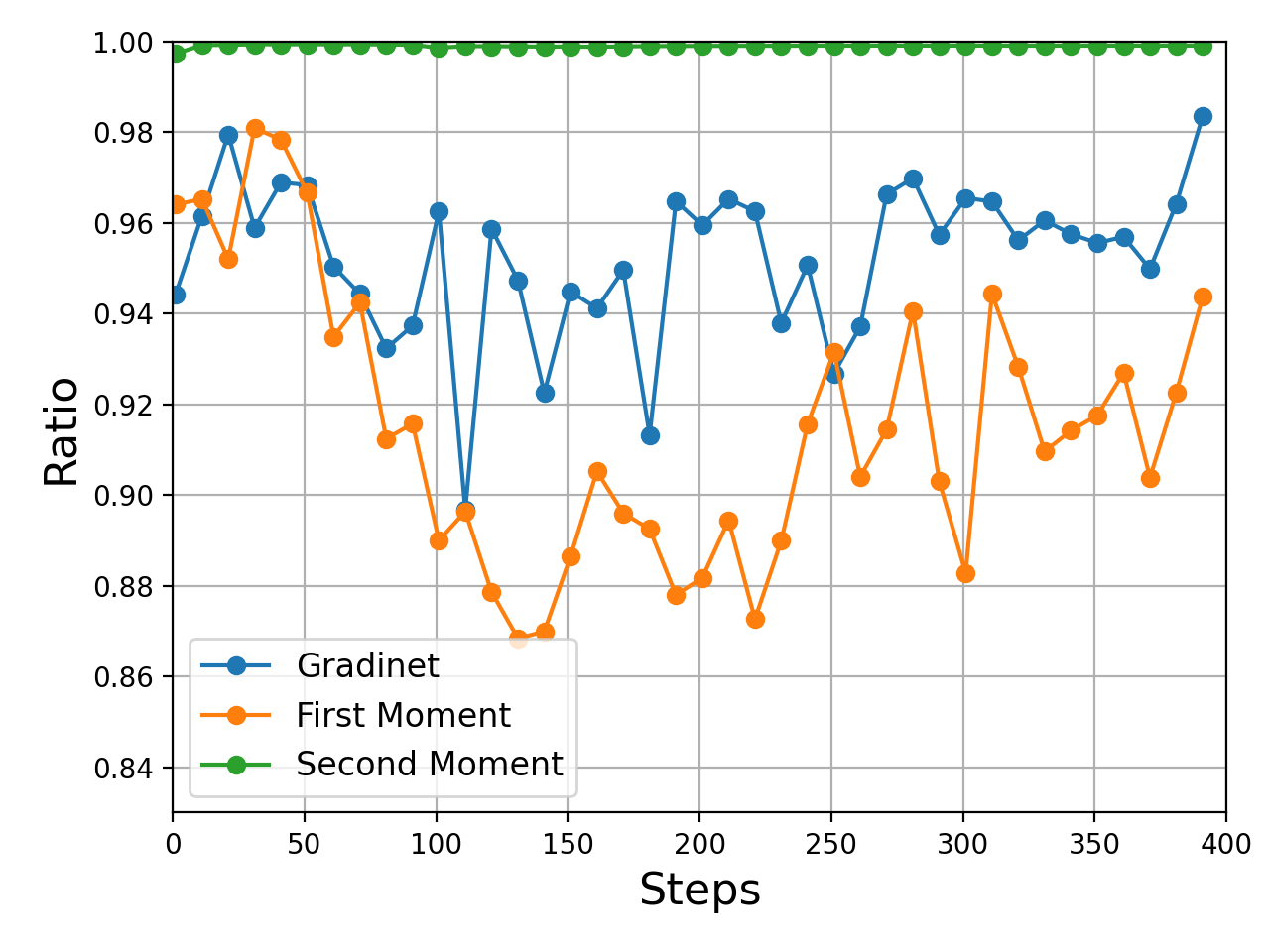}
        \caption{RTE}
    \end{subfigure}
    \begin{subfigure}[b]{0.4\textwidth}
        \centering
        \includegraphics[width=\textwidth]{fig/moment_stsb.png}
        \caption{STSB}
    \end{subfigure}

    \caption{Ratio of top-8 singular values to total singular values for gradient, first moment, and second moment during AdamW finetuning of RoBERTa-base on the CoLA, MRPC, RTE, STSB datasets.}
  \label{fig: moment}
\end{figure}

We conduct experiments examining the concentration of singular values in gradients and momenta during AdamW finetuning of RoBERTa-base on the CoLA, MRPC, RTE and STSB datasets. We set the batch size as 128, epochs as 20, learning rate as 1e-4 for all these four datasets. We use AdamW finetune the matrix parameters of query, key, value, output weights in attention layers and the intermediate and output weights in feed-forward layers. We conducted these experiments on NVIDIA RTX 6000 Ada GPUs. The average ratios of top-8 singular values to total singular values for gradient, first moment, and second moment of all these matrix parameters are reported in \Cref{fig: moment}. We can note that, in general, the momenta on all these datasets have highly concentrated singular values and exhibit low-rank structures, which is aligned with the intuition of our method.

\subsection{Memory Footprint with Per-layer Weight Updates}\label{app: memory}

As mentioned in \Cref{section: mlorc}, we can avoid storing full gradients in MLorc by using per-layer weight updates \citep{lv2024full}. Here we compare the memory consumption of LoRA and MLorc with per-layer weight updates with a batch size of 4.

\begin{table}[h]
\centering
\caption{Memory footprint of MLorc with per-layer weight updates and LoRA with a batch size of 4. Apart from batch size, other hyperparameters and settings are same as previous experiments.}
\label{tab:layer-wise}
\begin{tabular}{|c|c|}
\hline
MLorc (per-layer update) & LoRA\\
\hline
16.8GB & 17.7GB \\
\hline
\end{tabular}
\end{table}

Table \ref{tab:layer-wise} suggests that MLorc can even be more memory-efficient than LoRA with per-layer weight updates, which supports our claim in \Cref{section: mlorc}.

\subsection{Ablation study}

In MLorc-AdamW, we compress both the first moment and the second moment. In principle, we could instead compress only the first moment while retaining the original AdamW update rule for the second moment (denoted as MLorc\_m in \Cref{tab: ablation study}), or compress only the second moment while retaining the original AdamW update rule for the first moment (denoted as MLorc\_v in \Cref{tab: ablation study}). We conducted additional ablation experiments comparing MLorc\_m, MLorc\_v, and MLorc-AdamW . We adopted the same experimental settings as in the main paper, and the results are shown in \Cref{tab: ablation study}. We observe that the performance differences among MLorc\_m, MLorc\_v, and MLorc-AdamW are relatively small. MLorc achieves better results than both MLorc\_m and MLorc\_v on four out of the eight datasets. However, MLorc requires significantly less memory than MLorc\_m and MLorc\_v. For example, in the MRPC experiment, full fine-tuning consumes 2498 MB, while MLorc\_m and MLorc\_v require 2027 MB and 2026 MB, respectively. In contrast, MLorc-AdamW only uses 1703 MB. Therefore, we believe that MLorc-AdamW is more practical when memory resources are limited.

\begin{table}[h]
\centering
\caption{Ablation study on compressing different momenta in MLorc-AdamW}
\label{tab: ablation study}
\begin{tabular}{lccccccccc}
\toprule
Method & CoLA & MNLI & MRPC & QNLI & QQP & RTE & SST2 & STSB & Avg \\
\midrule
Full     & 62.33 & 87.62 & 91.11 & 92.92 & 90.26 & 75.81 & 95.18 & 90.50 & 85.72 \\ \midrule
MLorc-AdamW   & \textbf{62.07} & 87.53 & 90.77 & \textbf{93.19} & 88.99 & 77.98 & \textbf{95.18} & \textbf{90.59} & \textbf{85.79} \\
MLorc\_m (Only compress m)     & 61.07 & 87.51 & \textbf{91.34} & 93.19 & 88.99 & \textbf{78.70} & 95.18 & 90.59 & 85.69 \\
MLorc\_v (Only compress v)  & 59.14 & \textbf{87.67} & 91.31 & 92.38 & \textbf{90.32} & 75.81 & 95.07 & 90.58 & 85.29 \\
\bottomrule
\end{tabular}
\end{table}

\section{Detailed Experimental Settings}\label{app: setting}
\subsection{Fine-Tuning on MetaMathQA and CodeFeedback}\label{app: nlg}

The pre-trained LLaMA2-7B model is from Hugging
Face\footnote{https://huggingface.co/meta-llama/Llama-2-7b-chat-hf}. We have reported our batch size, epoch and other settings in \Cref{section: exp nlg}. Also, for all methods, on GSM8K dataset, the max sequence length is 512; on CodeFeedback dataset, the max sequence length is 1024; the weight decay is 0.  For GaLore, the subspace update frequency $T$ is set to 300 on both datasets. Oversampling parameter \(p\) is set as \(0\) for MLorc on both datasets. The temperature for evaluation is 0.8 for math task and 0.1 for coding task, since a high temperature would lead to highly unstable performance on HumanEval dataset. For each method and each dataset, the learning rate is individually tuned. We present specific learning rates in Table \ref{tab: Hyperparameter settings of NLG on LLaMA2-7B}.

\begin{table}[h]
\centering
\caption{Learning rates of different methods when fine-tuning LLaMA2-7B on MetaMathQA and CodeFeedback dataset.}
\begin{tabular}{lcccccccc}
\toprule
& MLorc-AdamW & Full (AdamW) & LoRA (AdamW) & GaLore & LDAdamW \\
\midrule
MetaMathQA     & 7E-05 & 4E-05 & 1E-03 & 3E-03 & 3E-04\\
CodeFeedback   & 7E-05 & 9E-05 & 3E-04 & 2E-03 & 3E-04\\

\bottomrule
\end{tabular}

\vspace{0.3cm}

\begin{tabular}{lcccccccc}
\toprule
& MLorc-Lion & Full (Lion) & LoRA (Lion)\\
\midrule
MetaMathQA     &1E-05 & 3E-05 & 1E-04 \\
CodeFeedback   &7E-06 & 2E-05 & 2E-04 \\

\bottomrule
\end{tabular}

\label{tab: Hyperparameter settings of NLG on LLaMA2-7B}
\end{table}

\subsection{Fine-Tuning on GLUE}\label{app: nlu}

The pre-trained RoBERTa-Base model is from Hugging
Face\footnote{https://huggingface.co/docs/transformers/model\_doc/roberta}. We use the same batch size, number of epochs, and maximum sequence length across all methods, including Full fine-tuning, MLorc, LoRA, and GaLore. For each method and each dataset, the learning rate is individually tuned. The LoRA scaling factor $\alpha$ is set to 16 for all tasks. For GaLore, the subspace update frequency $T$ is set to 50 for CoLA, MRPC, RTE, and STSB, and 100 for SST2, QNLI, MNLI, and QQP. We set the oversampling parameter \(p\)  as \(0\) for MLorc for all datasets. Experiments for CoLA, MRPC, and RTE are conducted on NVIDIA H100 GPUs; STSB, SST2, and QNLI are conducted on NVIDIA RTX A6000 GPUs; MNLI on NVIDIA RTX 6000 Ada GPUs; and QQP on NVIDIA GeForce RTX 3090 GPUs. Detailed hyperparameter settings are provided in \Cref{tab: Hyperparameter settings of GLUE}.

\begin{table}[h]
\centering
\caption{Hyperparameter settings for the GLUE tasks. "LR" denotes the learning rate.}
\begin{tabular}{lcccccccc}
\toprule
 & CoLA & MNLI & MRPC & QNLI & QQP & RTE & SST2 & STSB \\
\midrule
Batch Size     & 128 & 128 & 128 & 128 & 128 & 128 & 128 & 128 \\
Epochs         & 10  & 5   & 20  & 5   & 5   & 10  & 10  & 20  \\
Max Seq. Len.  & 64  & 256 & 64  & 256 & 256 & 256 & 128 & 128 \\
LR of Full     & 3E-05 & 3E-05 & 7E-05 & 1E-05 & 7E-05 & 3E-05 & 7E-06 & 1E-04 \\
LR of MLorc    & 3E-05 & 1E-04 & 7E-05 & 5E-05 & 7E-05 & 5E-05 & 5E-05 & 7E-05 \\
LR of LoRA     & 1E-03 & 3E-04 & 1E-03 & 5E-04 & 5E-04 & 7E-04 & 3E-04 & 5E-04 \\
LR of GaLore   & 3E-04 & 3E-04 & 5E-04 & 3E-04 & 3E-04 & 5E-04 & 3E-04 & 3E-04 \\
LR of LDAdamW  & 7E-05 & 7E-05 & 3E-04 & 7E-05 & 1E-04 & 5E-04 & 7E-05 & 1E-04 \\
\bottomrule
\end{tabular}
\label{tab: Hyperparameter settings of GLUE}
\end{table}

\subsection{License information}\label{app: lincese} LLaMA 2-7B \citep{touvron2023llama} is licensed under the LLaMA 2 Community License Agreement. CodeFeedback \citep{zheng2024opencodeinterpreter} is  licensed under Apache License 2.0. RoBERTa \citep{liu2019roberta}, MetaMathQA \citep{yu2023metamath}, GSM8K \citep{cobbe2021training}, HumanEval \citep{chen2021evaluating} are licensed under the MIT License.

\end{document}